\documentclass[twoside]{article}
\usepackage{mathrsfs}

\newif\ifisarxiv
\isarxivtrue

\ifisarxiv
\usepackage[accepted]{sty/aistats2019_arxiv}
\else
\usepackage{sty/aistats2019}
\fi
\usepackage{amsmath}
\usepackage{amssymb}
\usepackage{color}
\usepackage{cancel}
\usepackage{algorithm}
\usepackage{algorithmic}
\usepackage{graphicx}
\usepackage{hyperref}

\newcommand\mydots{\makebox[1em][c]{.\hfil.\hfil.}}
\def\S{{\mathbb{S}}}
\def\Sd{\mathscr{S}_{\!d}}
\newcommand{\dx}{D_{\!\cal X}} 
\newcommand{\dxk}{D_{\!\cal X}^k} 
\newcommand{\dk}{D^k} 
\newcommand{\dxy}{D}
\def\vskx{{\mathrm{VS}_{\!\dx}^k}}
\def\vsk{{\mathrm{VS}_{\!D}^k}}
\def\vskxm{{\mathrm{VS}_{\!\dx}^{k-1}}}
\def\vskm{{\mathrm{VS}_{\!D}^{k-1}}}
\def\vsdx{{\mathrm{VS}_{\!\dx}^d}}
\def\vsd{{\mathrm{VS}_{\!D}^d}}
\newcommand{\vs}[1]{{\mathrm{VS}_{\!D}^{#1}}}
\newcommand{\sigd}{\boldsymbol\Sigma_{\!\dx}}
\def\wds{\boldsymbol\w_{\!D}^*}
\def\kd{K_{\!\dx}}

\def\Vol{{\mathrm{VS}}}
\def\Lev{{\mathrm{Lev}}}

\newenvironment{proofof}[2]{\par\vspace{2mm}\noindent\textbf{Proof of {#1} {#2}}\ }{\hfill\BlackBox\\[2mm]}

\DeclareMathOperator{\sgn}{\textnormal{sgn}}

\def\Sigmab{\mathbf{\Sigma}}
\def\Sigmabh{\widehat{\Sigmab}}
\def\Sigmabt{\widetilde{\Sigmab}}

\def\xbt{\widetilde{\x}}

\def\Xt{\widetilde{X}}

\def\Nc{\mathcal{N}}

\ifx\BlackBox\undefined
\newcommand{\BlackBox}{\rule{1.5ex}{1.5ex}}  % end of proof
\fi

\DeclareMathOperator*{\argmin}{\mathop{\mathrm{argmin}}}

\def\x{\mathbf x}
\def\y{\mathbf y}

\def\yt{\widetilde{y}}

\def\a{\mathbf a}
\def\b{\mathbf b}
\def\w{\mathbf w}

\def\wbh{\widehat{\mathbf w}}

\def\zero{\mathbf 0}
\def\one{\mathbf 1}

\def\X{\mathbf X}

\def\B{\mathbf B}
\def\A{\mathbf A}
\def\C{\mathbf C}
\def\U{\mathbf U}

\def\V{\mathbf V}

\def\St{\widetilde{\S}}

\def\I{\mathbf I}

\def\A{\mathbf A}

\def\Xt{\widetilde{\mathbf{X}}}

\def\Ot{\widetilde{O}}

\def\E{\mathbb E}
\def\R{\mathbb R} 
\def\Pr{\mathrm{Pr}} 
\def\tr{\mathrm{tr}}

\newcommand{\defeq}{\stackrel{\textit{\tiny{def}}}{=}}

\let\origtop\top
\renewcommand\top{{\scriptscriptstyle{\origtop}}} % this makes transpose not so big

\definecolor{silver}{cmyk}{0,0,0,0.3}
\definecolor{yellow}{cmyk}{0,0,0.9,0.0}
\definecolor{reddishyellow}{cmyk}{0,0.22,1.0,0.0}
\definecolor{black}{cmyk}{0,0,0.0,1.0}
\definecolor{darkYellow}{cmyk}{0.2,0.4,1.0,0}
\definecolor{orange}{cmyk}{0.0,0.7,0.9,0}
\definecolor{darkSilver}{cmyk}{0,0,0,0.1}
\definecolor{grey}{cmyk}{0,0,0,0.5}
\definecolor{darkgreen}{cmyk}{0.6,0,0.8,0}
\newcommand{\Red}[1]{{\color{red}  {#1}}}

\newcommand{\white}[1]{{\textcolor{white}{#1}}}

\ifx\proof\undefined
\newenvironment{proof}{\par\noindent{\bf Proof\ }}{\hfill\BlackBox\\[2mm]}
\fi

\ifx\theorem\undefined
\newtheorem{theorem}{Theorem}
\fi

\ifx\example\undefined
\newtheorem{example}{Example}
\fi

\ifx\condition\undefined
\newtheorem{condition}{Condition}
\fi
\ifx\property\undefined
\newtheorem{property}{Property}
\fi

\ifx\lemma\undefined
\newtheorem{lemma}[theorem]{Lemma}
\fi

\ifx\proposition\undefined
\newtheorem{proposition}[theorem]{Proposition}
\fi

\ifx\remark\undefined
\newtheorem{remark}[theorem]{Remark}
\fi

\ifx\corollary\undefined
\newtheorem{corollary}[theorem]{Corollary}
\fi

\ifx\definition\undefined
\newtheorem{definition}{Definition}
\fi

\ifx\conjecture\undefined
\newtheorem{conjecture}[theorem]{Conjecture}
\fi

\ifx\axiom\undefined
\newtheorem{axiom}[theorem]{Axiom}
\fi

\ifx\claim\undefined
\newtheorem{claim}[theorem]{Claim}
\fi

\ifx\assumption\undefined
\newtheorem{assumption}[theorem]{Assumption}
\fi

\ifx\condition\undefined
\newtheorem{condition}[theorem]{Condition}
\fi

\begin{document}
\runningtitle{Correcting the bias in least squares regression}


\twocolumn[

\aistatstitle{Correcting the bias in least squares regression\\
with volume-rescaled sampling}

\ifisarxiv
\aistatsauthor{
  Micha{\l } Derezi\'{n}ski
  \And
  Manfred K. Warmuth
  \And
Daniel Hsu
}
\aistatsaddress{\small
Foundations of Data Analysis Institute\\
\small University of California, Berkeley\\
\small  \texttt{mderezin@berkeley.edu}
  \And
\small  Dept. of Computer Science\\
\small University of California, Santa Cruz\\
\small\texttt{manfred@ucsc.edu}
  \And
\small Dept. of Computer Science\\
\small Columbia University, New York\\
\small\texttt{djhsu@cs.columbia.edu}
}
\else
\aistatsauthor{ Author 1 \And Author 2 \And Author 3}
\aistatsaddress{ Institution 1 \And  Institution 2 \And Institution 3 }
\fi
]

\begin{abstract}
Consider linear regression where the examples are generated
by an unknown distribution on $\R^d\times\R$.
Without any assumptions on the noise, the linear least squares solution
for any i.i.d.~sample will typically be biased w.r.t.~the least squares
optimum over the entire distribution.  
However, we show that if an i.i.d.~sample of any
size $k$ is augmented by a certain small additional sample,
then the solution of the combined sample becomes unbiased.
We show this when the additional sample consists of $d$ points
drawn jointly according to the input distribution that is rescaled by the
squared volume spanned by the points. Furthermore, we propose
algorithms to sample from this volume-rescaled distribution when the
data distribution is only known through an i.i.d~sample.
\end{abstract}

\section{INTRODUCTION}
\label{s:introduction}
Unbiased estimators for linear regression are useful because
averaging such estimators 
gives an unbiased estimator whose prediction variance vanishes as the
number of averaged estimators increases. 
Such estimators might for example be produced in a distributed fashion from
multiple small samples. 
In this paper we develop a unique method for correcting the bias
of linear least squares estimators. Our main methodology 
for producing an unbiased estimator is
volume sampling. For a fixed design matrix $\X \in \R^{n\times d}$,
the most basic variant of volume sampling chooses a subset $S\subseteq
\{1..n\}$ of dimension many rows (i.e. $|S|=d$)
with probability proportional to the squared volume spanned by the
rows, i.e. $\det(\X_S)^2$, where $\X_S$ is the sub-matrix of rows
indexed by $S$. This procedure 
generalizes to sampling sets of any fixed size $k\geq d$
\cite{avron-boutsidis13}:
\begin{equation}
    \label{eq:vol}
P(S)\defeq \frac{\det(\X_S^\top\X_S)}{ \Red{n-d \choose k-d} \det(\X^\top\X)}
.
\end{equation}
Volume sampling has the property that for any design matrix
$\X$ with $n$ rows and any real response vector $\y\in\R^n$, the linear 
least squares solution for the subproblem $(\X_S,\y_S)$
is an unbiased estimator for the solution of the full problem
$(\X,\y)$ \cite{unbiased-estimates-journal}. 

We propose the following previously unobserved
alternate sampling method for size $k> d$ volume sampling:
First volume sample a set $S_\circ$ of size $d$ and then
pad the sample with a uniform subset $R$ of $k-d$
rows outside of $S_\circ$. Now the probability of the
combined size $k$ sample $S=S_\circ \cup R$ is again
volume sampling \eqref{eq:vol}:
$$
P(S)\!=\!\!\!\!
\sum_{\underset{|\hspace{-0.2mm}S_\circ\!|=d}{S_\circ\subseteq S}} \!\underbrace{P\big(R\!=\!S\!\setminus\! S_\circ\,|\,S_\circ\big)}
_{\Red{\frac{1} {{n-d \choose k-d}} }}
\!\!\!
\underbrace {P(S_\circ)} _ {\frac{\det(\X_{S_\circ}\!)^2}{\det(\X^\top\X)}}
\!\!\!=\! \frac{\det(\X_S^\top\X_S)}{ \Red{n-d \choose k-d} \!\det(\X^\top\X)},
$$
where the equality is the Cauchy-Binet formula for
determinants.
Furthermore, we study a more general statistical learning model where
the points come from an unknown probability
distribution over $\R^d\times \R$, and the goal is to recover the 
least squares solution w.r.t.~the distribution.
In this paper we generalize volume sampling to this case by rescaling
the i.i.d.~sampling distribution by the squared volume
of the sampled points.  

The simplest way to obtain a linear least squares estimator 
in the statistical learning model is to find the linear
least squares solution for a size $k$ i.i.d.~sample. 
Unfortunately such estimators are generally biased.
% w.r.t.~the optimum solution for the data distribution
Note that this
is not the kind of bias that we deliberately impose with
regularization to reduce the variance of the estimator. Rather, due to the random design, the
least squares estimator is typically biased even when it is not
regularized at all~\cite{hsu2014random}, and we have limited control over how large that
bias may be (see Section~\ref{s:experiment} for a motivating
example).
However our alternate sampling procedure for volume sampling
(discussed in the previous paragraph) implies the following strategy
for correcting the bias:
We show that if an i.i.d.~sample of any size $k$ is 
augmented with a size $d$ volume-rescaled sample for this
distribution, then the combined sample is a volume-rescaled sample of size
$k+d$, and its linear least squares solution is an unbiased estimator
of the optimum.
In one dimension, this means that if 
an i.i.d.~sample is augmented with 
just one example, where this additional example is drawn from a distribution 
whose marginal distribution on $x$ is proportional to 
the original (unknown) marginal density times $x^2$,
then the resulting least squares estimator
becomes unbiased. 
Curiously enough, for the purpose of correcting the bias
it does not matter whether the size $d$ volume-rescaled
sample was generated before or after the original size
$k$ i.i.d. sample was drawn, since they are independent of each other.

In addition to generalizing volume sampling to the
continuous domain and showing that only a subsample of size
$d$ needs to be rescaled by the squared volume,
we study the time and sample complexity of volume-rescaled sampling
when the data distribution is only known through an i.i.d.~sample.
Specifically:
\vspace{-3mm}
\begin{enumerate}
\item We extend \emph{determinantal rejection sampling}
  \cite{leveraged-volume-sampling} to  arbitrary data distributions with
  bounded support, and our improved analysis
  reduces its time and sample complexity by a factor of $d$.
\item When the data distribution is Gaussian with unknown covariance, we
  propose a new algorithm with $O(d)$ sample complexity.
\end{enumerate}

\paragraph{Related work.}
Discrete volume sampling of size $k\leq d$ was introduced to computer
science literature by \cite{pca-volume-sampling}, with later algorithms by
\cite{efficient-volume-sampling,more-efficient-volume-sampling}. The
extension to sets of size $k>d$ is due to \cite{avron-boutsidis13},
with algorithms by
\cite{dual-volume-sampling,unbiased-estimates-journal,leveraged-volume-sampling},
and additional applications in experimental design explored by
\cite{tractable-experimental-design,proportional-volume-sampling,symmetric-polynomials}. Our alternate volume sampling procedure implies that the
algorithms by
\cite{efficient-volume-sampling,more-efficient-volume-sampling} can be
used to volume sample larger sets at no additional cost. The
unbiasedness of least squares estimators under volume sampling was
explored by
\cite{unbiased-estimates-journal,leveraged-volume-sampling}, drawing
on observations of \cite{bental-teboulle}.

For arbitrary data distributions, volume-rescaled sampling of
size $d$ is a special case of a determinantal point
process (DPP) (see, e.g. \cite{dpp-statistics,dpp-first-algorithm}). 
However for $k>d$ and arbitrary distributions, 
we are not aware of such sampling appearing in the literature. 
Related variants of discrete DPPs have
been extensively explored in the machine learning community
\cite{dpp,k-dpp,k-dpp-efficient,dpp-shopping,celis2016fair,celis2018fair}.

\paragraph{Notations and assumptions.}
Throughout the paper, $(\x,y)\in\R^d\times \R$ 
is a random example drawn from some distribution $D$. We
assume that the point $\x$ and the response $y$ both have
finite second moments, i.e. $\E[\|\x\|^2]\!<\!\infty$ 
and $\E[y^2]\!<\!\infty$. The marginal probability measure
of $\x$ is denoted as $\dx$, while $\dxk$ is the probability
measure over $(\R^d)^k$ of $k$ i.i.d.~samples
$(\x_1, \textnormal{\fontsize{6}{6}\selectfont
\dots},\x_k)$ drawn from $\dx$. We define
$\sigd\defeq\E[\x\x^\top]\in\R^{d\times d}$ and w.l.o.g. assume that
it is invertible. Given a data sample $\S\!=\! \{(\x_1,y_1),
\textnormal{\fontsize{6}{6}\selectfont \dots},(\x_k,y_k)\}$, we denote
the least squares estimators for $\S$ and $\dxy$,
respectively, as
\begin{align*}
  \w^*(\S) &\defeq \argmin_\w \!\!\sum_{(\x_i,y_i)\in \S}
    (\x_i^\top\w-y_i)^2\quad\text{and}\\
  \wds &\defeq \argmin_\w \,\E_D\big[(\x^\top\w-y)^2\big] =
    \sigd^{-1}\E_D[\x\,y].
\end{align*}

\subsection{Statistical Results}

Our results are centered around the following size $k$ {\em joint sampling
distribution}.
\begin{definition}\label{t:cauchy-binet}
  % Given a random vector $\x$ drawn from a distribution
  %   $\dx$ on $\R^d$. Then 
    Given distribution $\dx$ and any $k\geq d$, we define
  volume-rescaled size $k$ sampling from $\dx$ as the
    following probability measure:
    For any event $A\subseteq (\R^d)^k$ measurable w.r.t. $\dxk$, its probability is
    \vspace{-1mm}
  \begin{align*}
    \vskx(A) \ &\defeq\  
	  \frac{\E_{\dxk}\Big[ \,\one_A\
               \overbrace{\det\!\Big(\sum\nolimits_{i=1}^k\!\x_i\x_i^\top\Big)}^
               {\textnormal{rescaling factor}}\Big]} 
	  {d! {k\choose d}\det\!\big(\sigd\big)},
  \end{align*}
  \vspace{-3mm}

  where $\one_A$ is the indicator variable of event $A$.
\end{definition} % When $\dx$ is finite and discrete, the above distribution corresponds % to one of the variants of \emph{volume sampling}, specifically the one
% recently proposed in \cite{leveraged-volume-sampling}, and it also
% bears a very close connection to the more widely studied variant of
% volume sampling defined originally by \cite{avron-boutsidis13} (called
% here \emph{standard volume sampling}, and discussed more in Section
% \ref{s:standard}). For arbitrary $\dx$, when $k=d$ the above
% definition is a special case of a determinantal point
% process (DPP) (see, e.g. \cite{bla,bla,bla}). For $k>d$ and arbitrary
% $\dx$, we are not aware of Definition \ref{t:cauchy-binet} previously
% appearing in the literature.
This distribution integrates to $1$ over its domain
$(\R^d)^k$ as a consequence of a continuous version of the classic
Cauchy-Binet formula, which has
appeared in the literature in various contexts (Lemma~\ref{l:determinant}).

Although we define volume-rescaled sampling for any sample size $k\geq d$,
we focus primarily on the special case of $k=d$
in the main results below. 
This is because we show that any $\vskx$
  can be decomposed into $\vsdx$ and $\dx^{k-d}$, the latter being the distribution of a size $k\!-\!d$ i.i.d.~sample from $\dx$.
\begin{theorem}\label{l:decomposition}
  Let $\S \sim \dx^{k-d}$ and $\S_\circ\sim\vsdx$. Let $\St\in\R^{k\times
    d}$ denote a random permutation of the points from $\S$
    concatenated with $\S_\circ$, i.e. $\St=\sigma(\langle
    \S,\S_\circ \rangle)$, where $\sigma$ is a
    random  permutation. Then $\St\sim\vskx$.
\end{theorem}
Given the above decomposition, one may wonder what is the purpose of
defining volume-rescaled sampling for any size $k>d$. In fact, we
will see in the following sections that both in the proofs and in
algorithms it is sometimes easier to work with $\vskx$ rather
than its decomposed version. For example in the theorem below, we
show that for any $k$, the least squares estimator computed on a
volume-rescaled sample is unbiased.
% For discrete volume sampling, several variants of this property have
% been shown recently
% \cite{unbiased-estimates-journal,leveraged-volume-sampling}, however,
Despite the fact that continuous determinantal point processes have
been studied extensively in the past, we were not able to find
this result for arbitrary $\dx$ in the
literature.
\begin{theorem}\label{t:unbiased}
    Consider the following distribution $\vsk$ on samples
    $\S=\{(\x_1,y_1),\dotsc,(\x_k,y_k)\}$
    of size $k$:
  \begin{align*}
&\textnormal{Sample}& &\x_{1},\dots,\x_{k}\ \sim\ \vskx,\\
      &\textnormal{Query}&&y_{i}\ \sim\ D_{{\cal Y}|\x=\x_{i}}\quad \forall_{i=1..k}.
\end{align*}
Then $\E_\vsk[\w^*(\S)] = \wds$.
\end{theorem}
Combining Theorems~\ref{l:decomposition} and \ref{t:unbiased}, we
conclude that an i.i.d.~sample only needs to be augmented
by a dimension-size volume-rescaled sample (i.e., $k=d$) 
so that the least squares estimator becomes unbiased. 
\begin{corollary}\label{t:augmenting}
  Let $\S \!=\! \{(\x_1,y_1),\textnormal{\fontsize{6}{6}\selectfont \dots},(\x_k,y_k)\}
  \!\overset{\textnormal{\fontsize{6}{6}\selectfont i.i.d.}}{\sim}\!
  D^k$, for any $k\geq 0$. Consider the following procedure:
\begin{align*}
&\textnormal{Sample}& &\xbt_{1},\dots,\xbt_{d}\ \sim\ \vsdx,\\
    &\textnormal{Query}&&\yt_{i}\ \sim\ D_{{\cal Y} |\x=\xbt_{i}}\quad \forall_{i=1..d}.
\end{align*}
Then for $\S_{\circ}=\{(\xbt_{1},\yt_{1}),\dots,(\xbt_{d},\yt_{d})\}$,
\begin{align*}
 \E\big[\w^*(\langle \S, \S_{\circ}\rangle)\big] &= \E_{\S\sim
  D^k}\!\big[\, \E_{\S_{\circ}\sim \vsd}[\,\w^*(\langle \S,
                                                  \S_{\circ}\rangle)\,]\, \big]\\[1mm]
\textnormal{(Theorem~\ref{l:decomposition})}\quad&=\E_{\St \sim \vs{k+d}}
\big[\,\w^*\!\big(\,\St\,\big)\,\big] \\
\textnormal{(Theorem~\ref{t:unbiased})}\quad  &= \wds.
% \wds &=\E\big[\w^*(\langle \S, \S_{\circ}\rangle)\big] 
%    \\&= \Red{\frac
%        {\E_{D^{n+d}} \big[\det(\sum_{i=1}^d \xbt_i\xbt_i^\top) 
% 	\;\w^*(\langle \S, \S_{\circ}\rangle)\big]}
%        {d! {k\choose d}\det(\sigd)}.
% 	    }
\end{align*}
\end{corollary}
\ifisarxiv\vspace{-2mm}

\fi
% When $\dx$ is a known discrete and finite distribution, the querying
% procedure reduces to standard size $d$ volume sampling
% \cite{avron-boutsidis13}, which admits a
% number of efficient algorithms
% \cite{unbiased-estimates-journal}.
To put the above result in context,
we note that in the fixed design case it was known
that a single volume sampled subset $\S$ of any size $k\geq d$ produces an
unbiased least squares estimator (see, e.g., \cite{unbiased-estimates-journal}).
However this required that all $k$
points be sampled jointly from this special distribution.
Thus, Corollary~\ref{t:augmenting}
says that volume sampling can be used to correct the bias in existing
i.i.d.~samples via sample augmentation 
(requiring labels/responses for only $d$ additional points from $\vsdx$).
This is important in active learning scenarios, where
samples from $\dx$ (unlabeled data) are cheaper than draws
from $D_{{\cal Y} |\x}$ (label queries). We also develop
methods for generating the small sample from $\vsdx$ only using
additional unlabeled samples from $\dx$ (see Section \ref{s:algs-overview}). 
Indeed, active learning was a motivation for volume sampling in previous works~\cite{unbiased-estimates-journal,leveraged-volume-sampling}.

\subsection{A Simple Gaussian Experiment}
\label{s:experiment}
The bias in least squares estimators is present even when input is a
standard Gaussian. As an example, we let $d=5$ and set:
\begin{align*}
\x^\top\! = (x_1,\dots,x_d)\overset{\text{i.i.d.}}{\sim} \Nc(0,1) ,\quad\ y =
  %\x^\top\v\! +\!
  \xi(\x)\! +\! \epsilon,
\end{align*}
where the response $y$ is
%of a linear function $\x^\top\v$,
a non-linear function $\xi(\x)$ plus independent white noise
$\epsilon$. Note that it is crucial that the response contains some
non-linearity, and it is something that one would
expect in real datasets. For the purposes of the experiment, we wish to make the
least squares solution easy to compute algebraically, so we choose the
following response model:
\ifisarxiv\vspace{-2mm}\fi
\begin{align*}
%\v = \begin{pmatrix}1\\ |\\ 1\end{pmatrix},\quad
  \xi(\x) =
  \sum_{i=1}^d x_i + \frac{x_i^3}{3},\quad \epsilon \sim \Nc(0,1).
\end{align*}
\ifisarxiv\vspace{-5mm}

\fi
We stress that there is nothing special about the choice of this
response model other than the fact that it contains a non-linearity
and it is easy to solve algebraically for $\wds$.
% Observe that $\sigd=\I$, so
% the least squares solution can be computed as:
% \begin{align*}
%   \wds \!&= \argmin_\w\E\big[(\x^\top\w-y)^2\big]= \sigd^{-1}\E[y\x] \\
%   &= \sum_{i=1}^d\E\Big[\Big(\frac13x_i^3 + x_i\Big)\x\Big] \!=\! 
% \begin{pmatrix}
% \E[\frac13 x_1^4+x_1^2]\\
% |\\
% \E[\frac13 x_d^4+x_d^2]
% \end{pmatrix} \!=\! \begin{pmatrix}2\\ |\\ 2\end{pmatrix}\!.
% \end{align*}
% Here the second to last equality uses the fact that the
% cross terms are 0 due to independence and the last equality
% follows form $\E[x^4]=3$ and $\E[x^2]=1$ for $\Nc(0,1)$
We now compare the bias of the least squares estimator produced
for this problem by i.i.d.~sampling of $k$ points, with that of an
estimator computed from $k-d$ i.i.d.~samples augmented by $d$ volume
samples (so that the total number of samples is the same in both
    cases). We used a special formula (Theorem~\ref{t:gaussian} below) to produce the volume-rescaled samples
    when $\dx$ is Gaussian.
Our strategy is to produce many such estimators
$\wbh_1,\dots,\wbh_T$ independently (e.g. by computing them in
parallel on separate machines), and look at
estimation error of the average of those estimators, i.e.
\ifisarxiv\vspace{-4mm}\fi
\begin{align*}
\text{estimation error:}\quad
  \bigg\|\Big(\frac1T\sum_{t=1}^T\wbh_t\Big)
  - \wds\bigg\|^2.
\end{align*}
\ifisarxiv\vspace{-5mm}

\fi
\begin{figure}
  \vspace{-1.5mm}
  \includegraphics[width=0.5\textwidth]{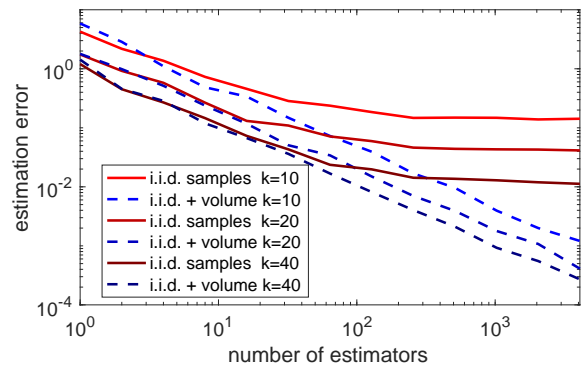}
  \ifisarxiv\vspace{-6mm}\fi
  \caption{Experiment with Gaussian inputs.}\label{f:experiment}
\end{figure}%
Figure~\ref{f:experiment} shows the above experiment for several
values of $k$ and a range of values of $T$ (each presented data point is
an average over 50 runs). 
Since the corrected estimator ``i.i.d.~+~volume'' is 
unbiased, the estimation error of the average estimator
exhibits $\frac1T$ convergence to zero (regardless of $k$). This type
of convergence
appears as a straight line on the log-log plot. In
contrast, the i.i.d.~sampled estimator
is biased for any sample size (although the bias decreases with $k$),
and therefore the averaged estimator does not converge to the optimum.

\subsection{Sampling Algorithms}
\label{s:algs-overview}
To our knowledge, existing literature on algorithms for DPPs and
volume sampling (other than the exceptions discussed below)
generally assumes full or considerable knowledge of the distribution
$\dx$, which often may not be the case in practice, for
example when the data is coming in a stream, or is drawn from a
larger population. In this work, we are primarily interested in
the setting where access to distribution $\dx$ is limited to some approximate
statistics plus the ability to draw i.i.d.~samples from
it. Two key concerns in this model are the time and sample complexities
of volume-rescaled sampling for a given distribution $\dx$.

%To be able to
%control the tails of the distribution, we have to assume
%that $\dx$
%has bounded support.
We first consider distributions $\dx$ with bounded support.
We use a standard 
notion of \emph{conditioning number} for multivariate distributions
(see, e.g.,~\cite{chen2017condition}): 
\begin{align*}
\kd \defeq \sup_{\xbt\in\text{supp}(\dx)} \xbt^\top\sigd^{-1}\xbt.
\end{align*}
When $\kd$ is known to be bounded and we are given the
exact knowledge of the covariance matrix $\sigd$, then it is
possible to produce a volume-rescaled sample $\vsdx$ using a
classical algorithm from the DPP literature described in
\cite{dpp-first-algorithm} by
employing rejection sampling (see also \cite{dpp-statistics}). This approach
requires $O(\kd\log(d))$ draws from $\dx$ and runs in time
$O(\kd d^2\log(d))$.
However, sampled sets produced by that
algorithm diverge from the desired distribution unless the given covariance
matrix matches the true one exactly. This
may be unrealistic when we do not have
full access to the distribution $\dx$. Is it possible to sample from
$\vsdx$ without the exact knowledge of $\sigd$?

We answer the question affirmatively.
We show that a recently proposed algorithm from \cite{leveraged-volume-sampling} for
fixed design volume sampling can be adapted to arbitrary $\dx$
in such a way that it only requires an approximation of the
covariance matrix $\sigd$, while still returning samples exactly
from $\vsdx$. The original algorithm,
called \emph{determinantal rejection sampling}, samples
from a given finite design matrix 
(i.e., a discrete distribution $\dx$ which is fully-known), but it was
shown in \cite{leveraged-volume-sampling} that the procedure only requires 
an approximation of the covariance matrix 
$\Sigmabh=(1\pm\epsilon)\sigd$, where $\epsilon=O(\frac1d)$. 
We extend this algorithm to handle arbitrary
distributions $\dx$, and also improve the analysis by reducing the required
approximation quality to $\epsilon = O(\frac1{\sqrt{d}})$.
\begin{theorem}\label{t:det} 
Given any $\Sigmabh\in\R^{d\times d}$ s.t.
  \begin{align*}
      &(1-\epsilon)\sigd\preceq\Sigmabh\preceq
    (1+\epsilon)\sigd,\;
      \\&\text{where}\;
      \epsilon=\frac1{\sqrt{2d}}\;\text{and}\;
  K\!\geq\! \frac{\kd}{1-\epsilon}, 
  \end{align*}
there is an algorithm which
returns $\xbt_1,\dots,\xbt_d\sim \vsdx$, and
with probability at least $1-\delta$ its sample and time complexity is
$O(K d(\ln(\frac1\delta))^2)$ and $O(K d^3(\ln(\frac1\delta))^2)$, respectively.
\end{theorem}
% Crucially, the algorithm requires that an estimate $\Sigmabh
% \approx\Sigmab$ of the covariance matrix is available,
% along with an upper-bound on the conditioning number. 
\begin{remark}\label{r:epsilon}
Our $\epsilon=\frac{1}{\sqrt{2d}}$ condition improves the result from
\cite{leveraged-volume-sampling} (where $\epsilon = \frac1{16d}$ was used). When $\dx$ is given as a finite set of $n$
vectors in $\R^d$, the main cost of volume sampling
is an $\Ot(nd + d^3/\epsilon^2)$ preprocessing step of computing
  $\Sigmabh$, where $\Ot(\cdot)$ hides $\operatorname{polylog}(n,d,1/\epsilon,1/\delta)$. Setting $\epsilon\!=\!\frac{1}{\sqrt{2d}}$ in Appendix~F of
\cite{leveraged-volume-sampling}, we reduce that cost from $\Ot(nd+d^5)$ to
$\Ot(nd+d^4)$.%Note that typically $d \ll n$.
\end{remark}
In Section~\ref{s:algorithm}, we discuss how $\Sigmabh$ can be obtained just
by sampling from the distribution $\dx$, which requires
$m=O(\kd
d\ln(d))$ samples with high probability and time
$O(md^2)=O(\kd d^3\ln(d))$, nearly the same (up to log terms)
as for the algorithm of Theorem~\ref{t:det} (here,
the improved $\epsilon$ also plays a key role).

The conditioning
number $\kd$ can be much larger than the dimension $d$ of the
distribution $\dx$, so obtaining an appropriate estimate of
$\sigd$ required for Theorem~\ref{t:det} may still be
prohibitively expensive. Thus, it is natural to ask if there are some
structural assumptions on distribution $\dx$ which can allow us to
sample from $\vsdx$ without any estimate of the covariance matrix.
In the following result, we exploit a connection between
volume-rescaled sampling and the Wishart distribution to show that when $\x$
is a centered multivariate normal, then without any knowledge of %the covariance matrix
$\sigd$, we can produce a volume-rescaled sample from only $2d+2$ samples
of $\dx$ and in $O(d^3)$ running time.
\begin{theorem}\label{t:gaussian}
  Suppose that the point distribution $\dx$ 
    is a Gaussian $\Nc(\zero, \sigd)$ and let
  $\x_1,\dots,\x_{2d+2} {\sim}\dx^{2d+2}$. 
  Then $\xbt_1,\dots,\xbt_d  \ \sim\ \vsdx$,
  where 
    \begin{align*}
\xbt_i \defeq
    \bigg(\sum_{j=d+1}^{2d+2}\x_j\x_j^\top\bigg)^{\frac12}
    \bigg(\sum_{j=1}^d\x_j\x_j^\top\bigg)^{-\frac12}\x_i.  \end{align*}
\end{theorem}
\textbf{Note.} %Square roots of matrices are not unique. So for the% theorem
For a positive definite matrix $\A$, we define $\A^{\frac12}$
as the unique lower triangular matrix with positive diagonal
entries s.t. $\A^{\frac12}(\A^{\frac12})^\top=\A$.

Finding other distribution families which allow for
volume-rescaled sampling with bounded sample complexity is an
interesting future research direction.

% \begin{theorem}\label{t:algorithm}
% Given $K\!>\!0$, there is an algorithm which for $\x$
% s.t. $\kd\!\leq\!
% K$ returns $\xbt_1,\dots,\xbt_d\sim\vsdx$,
% and with high probability has sample complexity $\Ot(K d^2)$ and
% time complexity $\Ot(K d^3)$.
% \end{theorem}

\section{SAMPLE AUGMENTATION}
\label{s:unbiased}
Let $\a_i^\top$ denote the $i$th row of a matrix $\A$. First, we
extend a classic lemma by
\cite{expected-generalized-variance}, which was originally used to show the expected value of a
metric in multivariate statistics known as ``generalized variance''. 
\begin{lemma}
  [based on \cite{expected-generalized-variance}]
  \label{l:determinant}
  If the (transposed) rows of the random matrices $\A,\B\!\in\!\R^{k\times d}$
  are sampled as pairs of vectors 
  $(\a_1,\b_1),\dots,(\a_k,\b_k)$ i.i.d.~from a distribution over
random vectors  $(\a,\b)\!\in\!\R^{d\times 2}$ such that $\E[\a\b^\top]$ exists, then
  % \overset{\textnormal{\fontsize{6}{6}\selectfont
  %     i.i.d.}}{\sim}(\x,\y)$,\mnote{What distr.?}
\begin{align*}
\E \big[\det(\A^\top\B)\big] &= d!{k\choose d}\det\!\big(\E[\a\b^\top]\big).
\end{align*}
\end{lemma}
The above result is slightly different than what was
presented in
\cite{expected-generalized-variance} (the original one had
$\A=\B$, and the sample mean was subtracted from the vectors
before constructing the matrix $\A^\top\A$), but the analysis is
similar (see proof in Appendix~\ref{a:unbiased}).
Note that for $\a=\b=\x$, Lemma~\ref{l:determinant} shows that
$\vskx$ integrates to $1$, making it a well-defined probability distribution:
\begin{align*}
  \E_{\dxk}\bigg[\det\!\Big(\sum_{i=1}^k\x_i\x_i^\top\Big)\bigg]
  = d!{k\choose d}\det(\sigd).
\end{align*}
The asymmetry of Lemma~\ref{l:determinant} is crucial
for showing the unbiasedness property of volume-rescaled sampling.
  \begin{proofof}{Theorem}{\ref{t:unbiased}}
  For $k=d$, the least squares estimator is simply the unique solution to a
  system of linear equations\footnote{Unless $\det(\X)=0$, in which
    case we let $\w^*(\S)=\X^+\y$.},
    % underdetermined, then $\det(\X)=0$ so such sample will never be
    % selected by $\vsd$, but for the sake of concretness we define $\w^*(\S)=\X^+\y$.},
    so Cramer's rule states that the $i$th
  component of that solution is given by:
\[\big(\w^*(\S)\big)_i =
  \frac{\det(\X\!\overset{i}{\leftarrow}\!\y)}{\det(\X)},\]
    where $\X\!\overset{i}{\leftarrow}\!\y$ is matrix $\X$
    with column $i$ replaced by $\y$. We first prove unbiasedness of
    $\w^*(\S)$ for samples of size $d$:
\begin{align*}
  \E_{\vsd}
    \big[\big(\w^*(\S)\big)_i\big] &=
\frac{ \E_{D^d}[\det(\X)^2
    \big(\w^*(\S)\big)_i]}{d!\det(\sigd)}\\
  &=\frac{\E_{D^d}\big[\det(\X)\det(\X\!\overset{i}{\leftarrow}\!\y)\big]}{d!\det(\sigd)}\\
\text{(Lemma~\ref{l:determinant})}\quad
&=\frac{\det\!\big(\,\E_D[\x\,(\x\!\overset{i}{\leftarrow}\!y)^\top]\,\big)}
           {\det(\sigd)}\\
    &= \frac{\det\!\big(\,\sigd \!\!\overset{i}{\leftarrow}\!\E_D[\x\,y]\,\big)}
            {\det(\sigd)}
=\big(\wds\big)_{\!i},
\end{align*}
where we applied Lemma~\ref{l:determinant} to the pair of
    $d\times d$ matrices $\A=\X$ and
$\B=\X \overset{i}{\leftarrow} \y$.
The case of $k>d$ follows by induction based on a formula shown in
\cite{unbiased-estimates-journal}:
\begin{align*}
  \E&_{\vsk}
    \big[\w^*(\S)\big]
  =\frac{\E_{\dk}\big[\det(\X^\top\X)
    \,\w^*(\S)\big]}{d!{k\choose d}\det(\sigd)}\\
&\!\overset{(1)}{=} \frac{\E_{\dk}
  \Big[\frac1{k-d}\sum_{i=1}^k
    \det(\X_{-i}^\top\X_{-i})\,\w^*\!\big(\S\backslash
  \{(\x_i,y_i)\}\big)\Big]}{d!{k\choose d}\det(\sigd)}\\
  &= \frac1{k-d}\frac{\sum_{i=1}^k \E_{\dk}\big[\det(\X_{-i}^\top\X_{-i})
    \w^*\!\big(\S\backslash\{(\x_i,y_i)\}\big)\big]}{d!{k\choose
    d}\det(\sigd)}\\
  &\!\overset{(2)}{=}\frac{k}{k-d}\,  \frac{d!{k-1\choose
    d}}{d!{k\choose
    d}}\E_{\vskm}\!\big[\w^*(\S)\big]=\E_{\vskm}\!\big[\w^*(\S)\big],
\end{align*}
where $\X_{-i}$ denotes matrix $\X$ without the $i$th row, $(1)$
follows from the formula shown in \cite{unbiased-estimates-journal}
(given in Lemma~\ref{l:pseudoinverse} of Appendix~\ref{a:unbiased}), while $(2)$ follows because the samples
$\x_1,\dots,\x_k\sim \dxk$ are exchangeable, i.e. 
$\x_1,\dots,\cancel{\x_i},\dots,\x_k$ is distributed identically to $\x_1,\dots,\x_{k-1}$.
\end{proofof}

Finally, our key observation given in Theorem~\ref{l:decomposition} is
that size $k$ volume-rescaled sampling can be decomposed into size $d$
volume-rescaled sampling plus i.i.d.~sampling of $k-d$ points. Note
that a version of this already occurs for discrete volume sampling (see
Section~\ref{s:introduction}). However it was not previously known
even in that case.
\begin{proofof}{Theorem}{\ref{l:decomposition}}
Let $\text{DVS}_{\dx}^k$ denote the distribution of a matrix
    $\X\in\R^{k\times d}$ whose transposed rows are 
    $\{\x_1,\dots,\x_k\}=\sigma(\langle \S ,\S_\circ\rangle)$. The
    probability of  a measurable event $A$ w.r.t.~$\text{DVS}_{\dx}^k$ is: 
\begin{align*}
\E_{\text{DVS}_{\dx}^k}\!&\big[\one_A\big] = \frac1{{k\choose d}}\sum_{T\subseteq [k]:\,|T|=d}
 \!\! \frac{\E_{\dxk}[\one_A\det(\X_T^\top\X_T)]}
                     { d! \det(\sigd)}\\
&\!\!=\frac1{d!{k\choose d}\! \det(\sigd)}\E_{\dxk}\!\bigg[\one_A \!\!\!
\sum_{T\subseteq [k]:\,|T|=d}\!\!\!
\det(\X_T^\top\X_T)\bigg]\\
  &\!\!\!\overset{(*)}{=}\!\frac1{d!{k\choose d}\!
    \det(\sigd)}\E_{\dxk}\!\big[\one_A
    \det(\X^\top\X)\big] \\
  &\!\!= \E_{\vskx}\![\one_A],
\end{align*}
where $[k]=\{1..k\}$, matrix $\X_T$ consists of the rows of $\X$ indexed by set $T$,
and $(*)$ follows from the Cauchy-Binet formula.
\end{proofof}

\section{VOLUME-RESCALED GAUSSIAN}\label{s:gaussian}
In this section, we obtain a simple formula for producing volume-rescaled
samples when $\dx$ is a centered multivariate Gaussian with any
(non-singular) covariance matrix. We achieve this by making a
connection to the Wishart
distribution. Thus, for this section, assume that
$\x\sim\Nc(\zero,\sigd)$, and let $\x_1,\dots,\x_k\sim
\dxk$ be the transposed rows of matrix $\X$. Then
matrix $\Sigmab=\X^\top\X\in\R^{d\times d}$ is distributed according to
Wishart distribution $W_d(k,\sigd)$ with $k$ degrees of
freedom. The density function of this random matrix is proportional to
$\det(\Sigmab)^{(k-d-1)/2}\exp(-\frac12\tr(\sigd^{-1}\Sigmab))$. On the
other hand, if $\Sigmabt = \Xt^\top\Xt$ is constructed from vectors
$\xbt_1,\dots,\xbt_k\sim\vskx$, then its density function is
multiplied by an additional $\det(\Sigmabt)$, thus increasing
the value of $k$ in the exponent of the determinant. This observation
leads to the following result:
\begin{theorem}\label{t:wishart}
If $\dx\sim \Nc(\zero,\sigd)$ and
    $\xbt_1,\dots,\xbt_k\sim\vskx$
are rows of a random matrix $\Xt\in\R^{k\times d}$, then
\begin{align*}
\Xt^\top\Xt\sim W_d(k+2,\sigd).
\end{align*}
\end{theorem}
\begin{proof}
  Let $\Sigmab=\X^\top\X \sim W_d(k,\sigd)$ and $\Sigmabt\sim
  W_d(k+2,\sigd)$. 
  For any measurable event $A$ over the random matrix $\Xt^\top\Xt$,
  we have
  \begin{align*}
   \Pr\big(\Xt^\top\Xt\!\in\! A\big) &=  \frac{\E[\one_{[\X^\top\X\in
    A]}\det(\X^\top\X)]}{\E[\det(\X^\top\X)]}\\
    &=\frac{\E[\one_{[\Sigmab\in A]}\det(\Sigmab)]}{\E[\det(\Sigmab)]}
      \overset{(*)}{=}\Pr\big(\Sigmabt\!\in\! A\big),
  \end{align*}
where $(*)$ follows because the density function of
Wishart distribution $\Sigmabt\sim W_d(k+2,\sigd)$ is proportional to
$\det(\Sigmabt) \det(\Sigmabt)^{(k-d-1)/2}\exp(-\frac12\tr(\sigd^{-1}\Sigmabt))$.
\end{proof}
This gives us an easy way to produce the total covariance matrix $\Xt^\top\Xt$ of
volume-rescaled samples in the Gaussian case. We next show that the
individual vectors can also be recovered easily.
\begin{proofof}{Theorem}{\ref{t:gaussian}}
  The proof relies on the following two lemmas.
  \begin{lemma}\label{l:conditional}
    For any $\Sigmab\in\R^{d\times d}$, the conditional distribution
    of $\Xt\sim\vskx$ given $\Xt^\top\Xt=\Sigmab$ is the same as the conditional
    distribution of $\X\sim \dxk$ given $\X^\top\X=\Sigmab$.
  \end{lemma}
  While this lemma (proven in Appendix~\ref{a:gaussian}) relies
  primarily on the definition of conditional
  probability, the second one uses properties of the matrix variate
  Beta and Dirichlet distributions.
  \begin{lemma}\label{l:beta}
    For $\Sigmab\!\in\!\R^{d\times d}$ and vectors $\x_1,\dots,\x_k\!\sim\!
    \Nc(\zero,\sigd)$ forming the transposed rows of a matrix $\X$, let
    \vspace{-3mm}
    \begin{align*}
      \xbt_i = \Sigmab^{\frac12}(\X^\top\X)^{-\frac12}\x_i.
    \end{align*}
    Then $\xbt_1,\dots,\xbt_k$ are jointly distributed as $k$ Gaussians
    $\Nc(\zero,\sigd)$ conditioned on $\Xt^\top\Xt=\Sigmab$. 
  \end{lemma}
  Putting Theorem~\ref{t:wishart} together with the two lemmas,
  we observe that for any $k\geq d$, constructing $\Sigmab\sim
  W_d(k+2,\sigd)$, and plugging it into Lemma~\ref{l:beta}, we
  obtain that $\xbt_1,\dots \xbt_k\sim \vskx$,
  completing the proof of Theorem~\ref{t:gaussian}.
\end{proofof}
We conclude this section with the proof of Lemma~\ref{l:beta}, which
demonstrates an interesting application for classical results in
matrix variate statistics.
\begin{proofof}{Lemma}{\ref{l:beta}}
Let $\Sigmab_1\sim
  W_d(k_1,\sigd)$ and $\Sigmab_2\sim
  W_d(k_2,\sigd)$ be independent Wishart matrices (where
  $k_1+k_2\geq d$). Then matrix
  \[\U = (\Sigmab_1\!+\!\Sigmab_2)^{-\frac12}\Sigmab_1\big((\Sigmab_1\!+\!\Sigmab_2)^{-\frac12}\big)^\top\]
  is matrix variate beta distributed, written as $\U\sim
  B_d(k_1,k_2)$. The following was shown by 
  \cite{matrix-variate-beta}:
  \begin{lemma}[\cite{matrix-variate-beta}, Lemma 3.5]\label{l:mvb} 
    If $\Sigmab\sim W_d(k,\sigd)$ is distributed independently of
    $\U\sim B_d(k_1,k_2)$, and if $k=k_1+k_2$, then
    \begin{align*}
      \B &= \Sigmab^{\frac12}\U\big(\Sigmab^{\frac12}\big)^\top
	\quad \text{and} \quad
      \C = \Sigmab^{\frac12}(\I-\U)\big(\Sigmab^{\frac12}\big)^\top
    \end{align*}
    are independently distributed and $\B\sim W_d(k_1,\sigd)$,
    $\C\sim W_d(k_2,\sigd)$.
  \end{lemma}
  Now, suppose that we are given a matrix $\Sigmab\sim
  W_d(k,\sigd)$. We can decompose it into components
  of degree one via
  a splitting procedure described in \cite{matrix-variate-beta},
  namely taking $\U_1\sim B_d(1,k\!-\!1)$ and computing
  $\B_1\sim \Sigmab^{\frac12}\U_1\big(\Sigmab^{\frac12}\big)^\top$,
  $\C_1=\Sigmab\!-\!\Sigmab_1$ as in
  Lemma~\ref{l:mvb}, then recursively repeating the procedure on
  $\C_1$ (instead of $\Sigmab$) with $\U_2\sim B_d(1,k\!-\!2)$, \ldots,
  until we get $k$  Wishart matrices of degree one summing to $\Sigmab$:
  \begin{align*}
\B_1 &= \Sigmab^{\frac12}\U_1\big(\Sigmab^{\frac12}\big)^\top\\
   \B_2 &=
          \underbrace{\Sigmab^{\frac12}(1\!-\!\U_1)^{\frac12}}_{\C_1^{1/2}}\U_2\underbrace{\big((\I\!-\!\U_1)^{\frac12}\big)^\top\big(\Sigmab^{\frac12}\big)^\top}_{(\C_1^{1/2})^\top}\\[-3mm]
    \vdots&\\
    \B_k&=\underbrace{\Sigmab^{\frac12}(1\!-\!\U_{k-1})^{\frac12}\dots}_{\C_{k-1}^{1/2}}\U_k\underbrace{\dots \big((1\!-\!\U_{k-1})^{\frac12}\big)^\top\big(\Sigmab^{\frac12}\big)^\top\!\!}_{(\C_{k-1}^{1/2})^\top}.
  \end{align*}
  The above collection of matrices can be described more simply via
  the matrix variate Dirichlet distribution. Given independent
  matrices $\Sigmab_i\sim W_d(k_i,\sigd)$ for $i=1..s$, the matrix
  variate Dirichlet distribution $D_d(k_1,\dots,k_s)$ corresponds to a
  sequence of matrices
  \begin{align*}
    \V_i=\Sigmab^{-\frac12}\Sigmab_i\big(\Sigmab^{-\frac12}\big)^\top, \  i=1..s, \quad
    \Sigmab=\sum_{i=1}^s\Sigmab_i.
  \end{align*}
  Now,  Theorem 6.3.14 from \cite{matrix-variate-distributions} states
  that matrices $\B_i$ defined recursively as above can 
  also be written as
  \begin{align*}
    \B_i = \Sigmab^{\frac12}\V_i\big(\Sigmab^{\frac12}\big)^\top,\quad
    (\V_1,\dots,\V_k)\sim D_d(1,\dots,1).
  \end{align*}
  In particular, we can construct them as
  \begin{align*}
    \B_i = \xbt_i\xbt_i^\top = \Sigmab^{\frac12}(\X^\top\X)^{-\frac12}\x_i
    \x_i^\top\big((\X^\top\X)^{-\frac12}\big)^\top\big(\Sigmab^{\frac12}\big)^\top\!.
  \end{align*}
  Note that since matrix $\Sigmab$ is independent of vectors $\x_i$, we
  can condition on it without altering the distribution of the
  vectors. It remains to observe that the conditional distribution of matrix
  $\B_i$ determines the distribution of $\xbt_i$ up to multiplying
  by $\pm 1$, and since both $\xbt_i$ and $-\xbt_i$ are identically
  distributed, we recover the correct distribution of
  $\xbt_1,\dots,\xbt_k$ conditioned on $\Xt^\top\Xt=\Sigmab$, completing the proof.
\end{proofof}

\section{GENERAL ALGORITHM}
\label{s:algorithm}
In this section, we present a general algorithm for volume-rescaled
sampling, which uses approximate leverage score sampling to generate a
larger pool of points from which the smaller volume-rescaled sample can be drawn. 
The method relies on a technique called ``determinantal rejection sampling'', 
introduced recently in \cite{leveraged-volume-sampling} 
for a variant of volume sampling of finite subsets of points from a fixed set. 
Also, as in \cite{leveraged-volume-sampling} our
algorithm uses the most standard volume sampling distribution 
(see \eqref{eq:vol} and the associated discussion in the introduction) as a subroutine which samples
a subset of points/rows from a fixed set.
This is done via an efficient implementation of
``reverse iterative sampling'' \cite{unbiased-estimates-journal} 
(See Algorithm \ref{alg:standard} for a high-level description of this
sampling method). Curiously enough, the efficient implementation
of reverse iterative sampling given by
\cite{unbiased-estimates-journal}
(denoted here as ``VolSamp$(\{\x_1,\dots,\x_n\},k)$'' 
and not repeated here for lack of space)
is again based on rejection sampling:
It samples a set of $k$ points out of $n$ in time $O(nd^2)$
(independent of $k$). The runtime bound for this implementation only holds with
high probability because of its use of rejection sampling.
% Also as a sub-procedure we use an implementation of the volume sampling distribution
% described in Section \ref{s:introduction}, which is the most standard volume
% sampling technique for choosing a subset of points from a
% fixed set. The underlying method for the sub-procedure is called 
% ``Reverse iterative volume sampling'' \cite{unbiased-estimates-journal}
% (denoted here as ``VolSamp$(\,\cdot\,,d)$'' and given in Appendix~\ref{a:algorithm}).
% It samples a set of $d$ points out of $k$ in time $O(kd^2)$ 
% (where the runtime bound holds with high probability \Red{???}). 

For our algorithm we assume that an estimate $\Sigmabh \approx\sigd$ 
of the covariance matrix is available, 
along with an upper-bound on the conditioning number.

\begin{algorithm}%[H]
%  {\fontsize{10}{10}\selectfont
%    \caption{}
  \caption{\bf \small 
    Determinantal rejection sampling\qquad\white{.}
\hspace{2.25cm}for arbitrary distributions $\dx$}
  \label{alg:det}
  \begin{algorithmic}[1]
    \STATE \textbf{Input:} $\Sigmabh, K, t$
%    \STATE $i\leftarrow 0$
    \STATE \textbf{repeat}
    \STATE \quad $k\rightarrow 0$
    \STATE \quad\textbf{while} $k<t$
    \STATE \quad\quad Sample $\x \sim \dx$
    \STATE \quad\quad $a\sim
    \text{Bernoulli}\Big(\min\big\{\,1,\ \frac{\x^\top\Sigmabh^{-1}\x}{K}\, \big\}\Big)$
    \STATE \quad\quad\textbf{if}
    $a\!=\!\text{true}$,\ \textbf{then}
    \STATE\quad\quad\quad $k\!\leftarrow\! k\!+\!1$
    \STATE \quad\quad\quad$\x_k\leftarrow \x$
    \STATE \quad\quad\quad$\xbt_k\leftarrow
    \frac{\sqrt{d}}{\sqrt{\x_k^\top\Sigmabh^{-1}\x_k}} \,\x_k$
    \vspace{-1mm}
    \STATE\quad\quad\textbf{end}
    \STATE \quad\textbf{end}
      \vspace{1mm} 
    \STATE \quad $\Sigmabt \leftarrow \frac1t\sum_{j=1}^t \xbt_j\xbt_j^\top$
%    \vspace{-1mm}
    \STATE \quad Sample $\textit{Acc}\sim
    \text{Bernoulli}\big(\min\{\,1,\ \det(\Sigmabt\Sigmabh^{-1})\,\}\big)$
    \vspace{-3mm}
    \STATE \textbf{until} $\textit{Acc}=\text{true}$
      \vspace{1mm} 
    \STATE $\{\xbt_{i_1},\dots,\xbt_{i_d}\} \leftarrow$
    VolSamp$\big(\{\xbt_1,\dots,\xbt_t\},d\big)$
    \RETURN $\x_{i_1},\dots,\x_{i_d}$
 \end{algorithmic}
%} 
\end{algorithm}

Algorithm~\ref{alg:det} has one additional hyperparameter $t$, which controls the number of inner-loop iterations. Our analysis works for any $t>d^2$, although for simplicity we use $t = 2d^2$ in the main result.

\begin{algorithm}[H] 
  \caption{\bf \ \small Reverse iterative sampling \cite{unbiased-estimates-journal}}
  \label{alg:standard}
  \begin{algorithmic}[1]
%    \STATE VolSamp$(\{\x_1,\dots,\x_n\},k)\!:$
    \STATE Input $\{\x_1,\dots,\x_n\}\subset \R^d$ and $k\geq d$
    \STATE \quad$S \leftarrow \{1..n\}$
\vspace{1mm}
    \STATE \quad{\bf while} $|S|>k$
    \STATE \quad\quad For each ${i\in S}$: $q_i\!\leftarrow\!
    \frac{\det(\sum_{j\in S\backslash i}\x_j\x_j^\top )}{(|S|-d)\det(\sum_{j\in S}\x_j\x_j^\top)}$
    \STATE \quad\quad Sample $i$ from  distribution $(q_i)_{i\in S}$
\vspace{1mm}
    \STATE \quad\quad $S\leftarrow S \backslash \{i\}$
    \STATE \quad{\bf end} 
    \RETURN $\{\x_i\}_{i\in S}$
  \end{algorithmic}
\end{algorithm}

Our analysis of
Algorithm~\ref{alg:det}
%satisfies the conditions of
%Theorem~\ref{t:det}, we will
uses the following two lemmas, both of which are
extensions of
results from~\cite{leveraged-volume-sampling}.
%(see proofs in Appendix~\ref{a:algorithm}).
\begin{lemma}\label{l:fast-rejection}
  For $\Sigmabh\succ 0$, let
  $l_{\Sigmabh}(\x)=\x^\top\Sigmabh^{-1}\x$. Define the following
  probability measure over $\R^d$:
  \begin{align*}
    \Lev_{\Sigmabh,{\cal X}}(A) \defeq \E_{\dx}\bigg[\,\one_A\ \frac{ l_{\Sigmabh}(\x)}
    {\tr(\sigd\Sigmabh^{-1})}\bigg].
  \end{align*}
If $\x_1,\dots,\x_t \overset{\textnormal{\fontsize{6}{6}\selectfont
    i.i.d.}}{\sim}\Lev_{\Sigmabh, {\cal X}}$, and
$\xbt_i=\frac{\sqrt{d}}{\sqrt{l_{\Sigmabh}(\x_i)}} \,\x_i$, then
\begin{align*}
  \det(\Sigmabt\Sigmabh^{-1})\leq 1,\quad\text{where}\quad
\Sigmabt=\frac1t\sum_{i=1}^t
  \xbt_i\xbt_i^\top,\\
  \text{and}\quad\E\big[\!\det(\Sigmabt\Sigmabh^{-1})\big]
  \geq
    \Big(1-\frac{d^2}{t}\Big)\frac{\det(\sigd\Sigmabh^{-1})}{(\frac1d\tr(\sigd\Sigmabh^{-1}))^d}.
\end{align*}
\end{lemma}
\begin{lemma}\label{l:composition}
Let $\x_1,\dots,\x_k\sim\vskx$ be a volume-rescaled sample,
and suppose that $\{\x_{i_1},\dots,\x_{i_d}\}$ is a subset
produced from it by standard volume sampling, i.e. by calling
$\textnormal{VolSamp}(\{\x_1,\dots,\x_k\},d)$. Then
    $\x_{i_1},\dots,\x_{i_d}\sim\vsdx$.
\end{lemma}
We now show that Algorithm~\ref{alg:det} with $t=2d^2$ satisfies the
conditions of Theorem~\ref{t:det}. Our key
contribution compared to the analysis of
\cite{leveraged-volume-sampling} is the use of the Kantorovich inequality,
which allows us to significantly relax the 
$\epsilon$-approximation condition on $\Sigmabh$. 
\begin{proofof}{Theorem}{\ref{t:det}}
From the assumptions, we have
\[K\geq \frac{\kd}{1-\epsilon}\geq
    \max_{\xbt\in\text{supp}(\dx)}\xbt^\top\Sigmabh^{-1}\xbt,\]
so the sequence $\xbt_1,\dots,\xbt_{t}$ obtained by the algorithm
at the point of exiting the \textbf{while} loop is
distributed as in Lemma~\ref{l:fast-rejection}, and let $D_{\cal \widetilde{\!X}}$ be the
distribution of one such vector. The lemma ensures that
$\det\!\big(\Sigmabt\Sigmabh^{-1}\big)\leq 1$ is a valid Bernoulli
success probability so after exiting the \textbf{repeat}
loop,
$\xbt_1,\dots,\xbt_{t}$ is distributed so that the probability of any
event $A$ is proportional to
\begin{align*}
\E_{D_{\cal
  \widetilde{\!X}}^t}\!\bigg[\one_A\frac{\det(\Sigmabt)}{\det(\Sigmabh)}\bigg]
  \propto \E_{D_{\cal
  \widetilde{\!X}}^t}\!\bigg[\one_A
  \det\!\Big(\sum_{i=1}^t\xbt_i\xbt_i^\top\Big)\bigg]
  \propto\Vol_{D_{\cal  \widetilde{\!X}}}^t,
  \end{align*}
  % $\det(\sum_i\xbt_i\xbt_i^\top)D_{\xbt}^t\propto\Vol_{\xbt}^{t}$,\mnote{?}
  i.e., volume-rescaled sampling from $D_{\cal \widetilde{\!X}}$.
Now Lemma \ref{l:composition} implies that
$\xbt_{i_1},\dots,\xbt_{i_d}\sim\Vol_{D_{\cal  \widetilde{\!X}}}^d$. In particular, it
means that the distribution of $\x_{i_1},\dots,\x_{i_d}$ is the
same for any choice of $t\geq d$. We use this observation to compute
the probability of an event $A$ w.r.t.~sampling of
$\x_{i_1},\dots,\x_{i_d}$ (up to constant factors) by setting $t=d$
(in the below, $\Sigmabt$ is treated as a function of
$\x_1,\dots,\x_d$): 
\begin{align*}
  \Pr(A) &\propto\E_{\dxk} \bigg[\,\one_A\,\det\!\big(\Sigmabt\big)
  \bigg(\prod_{i=1}^dl_{\Sigmabh}(\x_i)\bigg)\bigg]\\
  &\overset{(*)}{=} \E_{\dxk} \bigg[\,\one_A\,\frac{\det(\sum_i\x_i\x_i^\top)}{(\frac{d}{t})^d
    \prod_il_{\Sigmabh}(\x_i)}\bigg(\prod_{i=1}^dl_{\Sigmabh}(\x_i)\bigg)\bigg]\\
       &\propto \E_{\dxk}
         \bigg[\,\one_A\,\det\!\bigg(\sum_{i=1}^d\x_i\x_i^\top\bigg)\bigg]\\
  &\propto\vsdx(A),
\end{align*}
where $(*)$ uses the fact that for $t=d$, $\det(\Sigmabt)$ is the
squared volume of the parallelopiped spanned by $\x_1,\dots,\x_d$ and
stretched with the appropriate scaling factors.
Thus, we established the correctness of Algorithm~\ref{alg:det} for
any $t\geq d$, and we move on to complexity analysis. If we think of each iteration of the
\textbf{repeat} loop as a single Bernoulli trial, the success
probability $\Pr(\textit{Acc}\!=\!\text{true})$ equals
$\E[\det(\Sigmabt\Sigmabh^{-1})]$ with the expectation defined as in
Lemma~\ref{l:fast-rejection}. Let $\lambda_1,\dots,\lambda_d$ be the
eigenvalues of matrix $\Sigmabh \sigd^{-1}$. The approximation
guarantee for $\Sigmabh$ implies that all of these eigenvalues lie in
the range $[1\!-\!\epsilon,1\!+\!\epsilon]$. To lower-bound the
success probability, we use the Kantorovich arithmetic-harmonic mean
inequality. Letting $A(\cdot)$, $G(\cdot)$ and $H(\cdot)$ denote the
arithmetic, geometric and harmonic means respectively:
\begin{align*}
&\hspace{-5mm}\frac{\det(\sigd\Sigmabh^{-1})}{(\frac{1}{d}\tr(\sigd\Sigmabh^{-1}))^d}
  =\frac{\prod_{i=1}^d\frac{1}{\lambda_i}}{(\frac{1}{d}\sum_{i=1}^d\frac{1}{\lambda_i})^d}\\
  &=\bigg(\frac{H(\lambda_1,\dots,\lambda_d)}{G(\lambda_1,\dots,\lambda_d)}\bigg)^d
\;\overset{(1)}{\geq}\;\; \!\bigg(\frac{H(\lambda_1,\dots,\lambda_d)}{A(\lambda_1,\dots,\lambda_d)}\bigg)^d\\
  &\overset{(2)}{\geq}\!\big((1\!-\!\epsilon)(1\!+\!\epsilon)\big)^d
    \;\;\stackrel{\epsilon = 1/\sqrt{2d}} {=} \;\;\Big(1-\frac{1}{2d}\Big)^d
\!\geq \frac12,
\end{align*}
where $(1)$ is the geometric-arithmetic mean inequality and $(2)$ is
the Kantorovich inequality (\cite{kant}) with
$a\!=\!1\!-\!\epsilon$ and $b\!=\!1\!+\!\epsilon$:
$$\text{For }
0\!<\! a\!\le\! \lambda_1 , \!\mydots, \lambda_d \!\le\! b,\ \ 
\frac{A(\lambda_1,\!\mydots,\lambda_d)}{H(\lambda_1,\!\mydots,\lambda_d)}
\!\leq\! \bigg(\frac{A(a,b)} {G(a,b)}\bigg)^{\!2} \!\!.
$$
Now setting $t=2d^2$ in Lemma~\ref{l:fast-rejection}, we obtain that
\begin{align*}
\Pr(\textit{Acc}\!=\!\text{true}) =
  \E\big[\!\det(\Sigmabt\Sigmabh^{-1})\big]
  \geq \Big(1-\frac{d^2}{t}\Big)\,\frac12=\frac14.
\end{align*}
So a simple tail bound on a geometric random variable shows
that the number of iterations of \textbf{repeat} loop is $r\leq
\ln(\frac1\delta)/\ln(\frac34)$ w.p. at least $1-\delta$.
It remains to bound the number of samples needed from $\dx$. Note
that we can lower bound this success probability
\begin{align*}
  \Pr(a\!=\!\text{true})&=\E_{\dx}\bigg[\frac{\x^\top\Sigmabh^{-1}\x}{K}\bigg]
  = \frac{\tr(\sigd\Sigmabh^{-1})}{K}\\
  &\geq\frac{\tr(\sigd\sigd^{-1})}{(1+\epsilon) K}
    =\frac{d}{(1+\epsilon)K}. 
\end{align*}
Similarly as before we conclude that the number of samples needed for
a single iteration of \textbf{repeat} loop is $O(2d^2
\frac{K}{d}\ln(\frac1\delta))=O(Kd\ln(\frac1\delta))$ w.p. at least
$1-\delta$. Note that the computational cost per sample is $O(d^2)$
and the cost of VolSamp is $O(d^4)$, obtaining the desired complexities.
\end{proofof}
Finally, we discuss the time and sample complexity of obtaining 
$\Sigmabh$ with desired accuracy under the model where access to
$\dx$ is given only through sampling from the distribution. For this
we can rely on standard matrix Chernoff 
bounds given by \cite{matrix-tail-bounds}. The below version is
adapted from \cite{chen2017condition}: 
\begin{lemma}[\cite{matrix-tail-bounds,chen2017condition}]\label{l:matrix-tail} \ 
  If $\x_1,\x_2,\dots,\x_m \overset{\textnormal{\fontsize{6}{6}\selectfont
    i.i.d.}}{\sim}\dx$ and $m\geq C
\frac{\kd}{\epsilon^2}\ln(\frac{d}{\delta})$ for some absolute constant $C$, then
\begin{align*}
    (1-\epsilon)\sigd\preceq\frac1m\sum_{i=1}^m\x_i\x_i^\top\preceq
    (1+\epsilon)\sigd\quad\!\!\text{w.p.}\geq 1-\delta.
\end{align*}
\end{lemma}
Setting $\epsilon=\frac1{\sqrt{2d}}$ in
Lemma~\ref{l:matrix-tail}, we
note that the sample complexity of obtaining $\Sigmabh$ that would
  satisfy the assumptions of Theorem~\ref{t:det} is
  $m=O(\kd
  d\ln(\frac{d}{\delta}))$, and computing it takes $O(md^2) =
  O(\kd d^3\ln(\frac{d}{\delta}))$. % However, it may often be
%   faster to use a sketching technique called Fast Johnson-Lindenstraus
%   Transform \cite{ailon2009fast}, as described in
% \cite{fast-leverage-scores}, which can approximately compute
% $\Sigmabh$ (to the same level of accuracy $\epsilon=\frac1{4d}$) in
% time $O(md\ln(r))$, where $r=O(d^3\ln(md)\ln(d^3\ln(md)))$, obtaining
% time complexity $\Ot(\kd d^3)$.
% Note that the sample complexity cost of estimating $\Sigmabh$ dominates
% that of  performing volume sampling itself. This is due to the fact
% that the matrix Chernoff bounds offer a 
% relatively slow convergence rate of $O(\frac{1}{\epsilon^2})$, and
% Algorithm \ref{alg:det} requires a high accuracy guarantee on $\Sigmabh$ with
% $\epsilon=O(\frac{1}{d})$.
% Further improvements on the overall sample complexity
% of volume-rescaled sampling is an interesting future research direction.
\section{CONCLUSIONS}
We show that for the least squares estimator, the bias which occurs
in random design linear regression can be corrected by 
augmenting the dataset with dimension many points sampled from a
special joint distribution - an extension of discrete volume sampling. We
present two methods for performing this augmentation when the
underlying data distribution is only known through i.i.d.~samples. In
the process we improve the time complexity of a recently proposed
algorithm for discrete volume sampling.

An important future research direction is providing a random design error
analysis for the least squares estimator of the augmented
sample. Furthermore, it is natural to ask if  there are distribution families other than multivariate
normal which offer better complexity guarantees for producing
volume-rescaled samples.

\bibliographystyle{plain}
\bibliography{pap}

\clearpage
\appendix

\section{SAMPLE AUGMENTATION: PROOFS}
\label{a:unbiased}
In this section we give the proofs omitted in Section~\ref{s:unbiased}.

\begin{proofof}{Lemma}{\ref{l:determinant}}
First, suppose that $k=d$, in which case $\det(\A^\top\B) =
\det(\A)\det(\B)$. 
Recall that by  definition the determinant can be written as:
  \begin{align*}
\det(\C) = \sum_{\sigma\in \Sd}\sgn(\sigma)\prod_{i=1}^dc_{i,\sigma_i},
\end{align*}
where $\Sd$ is the set of all permutations of $(1..d)$, and
$\sgn(\sigma) = \sgn\big((1..d),\sigma\big)$ is the number of swaps
from $(1..d)$ to $\sigma$. Using this formula and denoting  $c_{ij} =
\big(\E[\a\b^\top]\big)_{ij}$, we can rewrite the expectation as: 
\begin{align*}
\E\big[\! \det(\A)\det(\B)\big]\!
&=\!\!\!\sum_{\sigma,\sigma'\in \Sd}\!\!\!\sgn(\sigma)\sgn(\sigma')
\prod_{i=1}^d \E\big[a_{i\sigma_i}b_{i\sigma'_i}\big]
\\
&= \sum_{\sigma\in \Sd}\sum_{\sigma'\in \Sd}\sgn(\sigma,\sigma') \prod_{i=1}^d c_{\sigma_i\sigma'_i}\\
  &=d! \sum_{\sigma'\in \Sd}\sgn(\sigma') \prod_{i=1}^d c_{i\sigma'_i}\\
  &=d!\det\!\big(\E[\a\b^\top]\big),
\end{align*}
which completes the proof for $k=d$. The case of
$k>d$ follows by induction via a standard determinantal formula:
\begin{align*}
\E\big[\det(\A^\top\B)\big] &\overset{(*)}{=}
  \E\bigg[\frac1{k-d}\sum_{i=1}^k\det\!\big(\A_{-i}^\top\B_{-i}\big)\bigg]\\
  &=\frac{k}{k-d}\,\E\big[\det\!\big(\A_{-k}^\top\B_{-k}\big)\big],
\end{align*}
where $(*)$ follows from the Cauchy-Binet formula and $\A_{-i}$
denotes matrix $\A$ with the $i$th row removed.
\end{proofof}
Next, we state a formula which we used in the proof
of Theorem~\ref{t:unbiased}. This lemma is an immediate implication of
a result shown by \cite{unbiased-estimates-journal}. 
\begin{lemma}\label{l:pseudoinverse}
  Given full rank $\X\in\R^{k\times d}$ and $\y\in\R^k$, we have:
  \begin{align*}
    \w^*(\X,\y) =
    \sum_{i=1}^k\frac{\det(\X_{-i}^\top\X_{-i})}{(k-d)\det(\X^\top\X)}\w^*(\X_{-i},\y_{-i}),
  \end{align*}
      where $\w^*(\X,\y) = \X^+\y$ is the least
      squares solution for $(\X,\y)$, and $\X^+$ is the pseudoinverse of $\X$.
    \end{lemma}
    \begin{proof}
Let $\I_{-i}$ denote the identity matrix with $i$th diagonal entry set
to zero. Note that we can write
$\w^*(\X_{-i},\y_{-i})=(\I_{-i}\X)^+\y$. Moreover, by Sylvester's
theorem we have
\begin{align*}
  \frac{\det(\X_{-i}^\top\X_{-i})}{\det(\X^\top\X)} =
  1-\x_i^\top(\X^\top\X)^{-1}\x_i.
\end{align*}
Thus, it suffices to show that
\begin{align*}
  \X^+ = \sum_{i=1}^k\frac{1-\x_i^\top(\X^\top\X)^{-1}\x_i}{k-d}(\I_{-i}\X)^+,
\end{align*}
which is in fact precisely the formula shown in
\cite{unbiased-estimates-journal} (see proof of Theorem 5).
\end{proof}

\section{VOLUME-RESCALED GAUSSIAN: PROOFS}\label{a:gaussian}
In this section we give the proofs omitted in Section~\ref{s:gaussian}.

\begin{proofof}{Lemma}{\ref{l:conditional}}
  Since we are conditioning on an event which may have probability
  $0$, this requires a careful limiting argument. Let $A$ be any
  measurable event over the random  matrix $\Xt$ and let 
  \begin{align*}
    C_\Sigmab^\epsilon \defeq \big\{\B\in\R^{d\times d}\,:\, \|\B-\Sigmab\|\leq \epsilon\big\}
  \end{align*}
  be an $\epsilon$-neighborhood of $\Sigmab$ w.r.t.~the matrix $2$-norm.
  We write the conditional probability of $\Xt\in A$ given that
  $\Xt^\top\Xt\in C_\Sigmab^\epsilon$ as:
  \begin{align*}
    \Pr\big(\Xt\!\in\! A\,|&\,\Xt^\top\Xt\!\in\! C_\Sigmab^\epsilon\big) =
    \frac{\Pr\big(\Xt\!\in\! A\,\wedge\,\Xt^\top\Xt\!\in\! C_\Sigmab^\epsilon\big)}
{\Pr\big(\Xt^\top\Xt\!\in\! C_\Sigmab^\epsilon\big)}\\
    &=\frac{\E\big[\one_{[\X\in A]}\one_{[\X^\top\X\in
      C_\Sigmab^\epsilon]}\det(\X^\top\X)\big]}
      {\E\big[\one_{[\X^\top\X\in
      C_\Sigmab^\epsilon]}\det(\X^\top\X)\big]}\\
&\leq      \frac{\E\big[\one_{[\X\in E]}\one_{[\X^\top\X\in
      C_\Sigmab^\epsilon]}\det(\Sigmab)(1+\epsilon)^d\big]}
      {\E\big[\one_{[\X^\top\X\in
                                              C_\Sigmab^\epsilon]}\det(\Sigmab)(1-\epsilon)^d\big]}\\
 &=\frac{\E\big[\one_{[\X\in A]}\one_{[\X^\top\X\in
      C_\Sigmab^\epsilon]}\big]}{\E\big[\one_{[\X^\top\X\in
   C_\Sigmab^\epsilon]}\big]} \bigg(\frac{1+\epsilon}{1-\epsilon}\bigg)^d\\
&=\Pr\big(\X\!\in\! A\,|\,\X^\top\X\!\in\! C_\Sigmab^\epsilon\big)
                                                                         \bigg(\frac{1+\epsilon}{1-\epsilon}\bigg)^d\\
                                                                   &\overset{\epsilon\rightarrow 0}{\longrightarrow}
                                                                     \Pr\big(\X\!\in\! A\,|\,\X^\top\X\!=\!\Sigmab\big).
  \end{align*}
  We can obtain a lower-bound analogous to the above upper-bound,
  namely
$\Pr\big(\X\!\in\! A\,|\,\X^\top\X\!\in\! C_\Sigmab^\epsilon\big)
    \big(\frac{1-\epsilon}{1+\epsilon}\big)^d$, which also converges
    to $\Pr\big(\X\!\in\! A\,|\,\X^\top\X\!=\!\Sigmab\big)$.
  Thus, we conclude that:
  \begin{align*}
    \Pr\big(\Xt\!\in\! A\,|\,\Xt^\top\Xt\!=\!\Sigmab\big) &=
    \lim_{\epsilon\rightarrow 0}\,\Pr\big(\Xt\!\in\!
    A\,|\,\Xt^\top\Xt\!\in\! C_\Sigmab^\epsilon\big)\\
    &=\Pr\big(\X\!\in\! A\,|\,\X^\top\X\!=\!\Sigmab\big),
  \end{align*}
  completing the proof.
  % Let $\B\in\R^{k\times d}$ be some fixed matrix s.t.
%   $\Sigmab=\B^\top\B$. Using the definition of conditional probability,
%   \begin{align*}
%     D_{\Xt}(\B\,|\,\Xt^\top\Xt\!=\!\Sigmab) &=
%     \frac{D_{\Xt}(\B)}{D_{\Xt^\top\Xt}(\Sigmab)} 
% =\frac{\det(\B^\top\B)\,D_\X(\B)} {\det(\Sigmab)\,D_{\X^\top\X}(\Sigmab)}\\
% &=\frac{D_\X(\B)}{D_{\X^\top\X}(\Sigmab)} = D_\X(\B\,|\,\X^\top\X\!=\!\Sigmab),
%   \end{align*}
%   which completes the proof.
\end{proofof}

\section{GENERAL ALGORITHM: PROOFS}\label{a:algorithm}
In this section we give proofs omitted in Section~\ref{s:algorithm}.
\begin{proofof}{Lemma}{\ref{l:fast-rejection}}
  The distribution $\Lev_{\Sigmabh,{\cal X}}$ integrates to one
  because for $\x\sim \dx$:
  \begin{align*}
    \E\big[\x^\top\Sigmabh^{-1}\x\big]
    = \E\Big[\tr\big(\x\x^\top\Sigmabh^{-1}\big)\Big]
    = \tr\big(\sigd\Sigmabh^{-1}\big).
  \end{align*}
Next, we use the geometric-arithmetic mean
inequality for the eigenvalues of matrix $\Sigmabt$ to show that:
\begin{align*}
  \det\!\big(\Sigmabt\Sigmabh^{-1}\big)
  &\leq\Big(\frac{1}{d}\tr\big(\Sigmabt\Sigmabh^{-1}\big)\Big)^{\!d}\\
  &=\Big(\frac{1}{d\,t}
    \sum_{i=1}^t\frac{d}{l_{\Sigmabh}(\x_i)}\x_i^\top\Sigmabh^{-1}\x_i\Big)^d=1.
\end{align*}
Next, we use the formula for the normalization constant in Theorem~\ref{t:cauchy-binet} but with a modified random vector. Specifically, let
$\xbt_i=\frac{\tr(\sigd\Sigmabh^{-1})}{l_{\Sigmabh}(\x_i)}\x_i$. Then
$\E[\xbt_i\xbt_i^\top]=\sigd$ and
\begin{align*}
\Sigmabt =
  \frac{1}{t}\sum_{i=1}^t\frac{d}{l_{\Sigmabh}(\x_i)}\x_i\x_i^\top =
  \frac{d}{\tr(\sigd\Sigmabh^{-1})}\frac{1}{t}\sum_{i=1}^t\xbt_i\xbt_i^\top.
\end{align*}
So, using Lemma~\ref{l:determinant} on the vectors $\xbt_i$, we have:
\begin{align*}
  \E\big[\!\det(\Sigmabt\Sigmabh^{-1})\big]
  &=\bigg(\frac{d}{\tr(\sigd\Sigmabh^{-1})}\bigg)^{\!d}\,
    \frac{\E[\det(\sum_{i}\xbt_i\xbt_i^\top)]}{t^d\det(\Sigmabh)}\\
  &= \frac{d!{t\choose
  d}\det\!\big(\E[\xbt_1\xbt_1^\top]\big)}
    {t^d (\frac{1}{d}\tr(\sigd\Sigmabh^{-1}))^d\det(\Sigmabh)}\\
  &=\bigg(\prod_{i=0}^{d-1}\frac{t-i}{t}\bigg)
    \frac{\det(\sigd\Sigmabh^{-1})}
    {(\frac{1}{d}\tr(\sigd\Sigmabh^{-1}))^d}\\
  &\geq \bigg(1-\frac{d}{t}\bigg)^d \frac{\det(\sigd\Sigmabh^{-1})}
    {(\frac{1}{d}\tr(\sigd\Sigmabh^{-1}))^d}.
\end{align*}
Applying Bernoulli's inequality concludes the proof.
\end{proofof}

\begin{proofof}{Lemma}{\ref{l:composition}}
% We first recall the reverse iterative sampling ``VolSamp''
% of \cite{unbiased-estimates-journal} in Algorithm~\ref{alg:standard}
% (for the sake of clarity, this pseudocode hides certain implementation
% details necessary for making the algorithm efficient).
Let $\X\in\R^{k\times d}$ be the matrix with rows $\x_i^\top$ and let
$q_i(\X)$ denote the sampling probability in line 4 of Algorithm~\ref{alg:standard},
given the set of row vectors. We will show that if
$\x_1,\dots,\x_k\sim\vskx$, then after one step of the algorithm, the
remaining vectors are distributed according to
$\vskxm$. Let $A$ denote a measurable event over the space
$(\R^d)^{k-1}$, and let $A' = A\times \R^d$ be that event marginalized
over the space $(\R^d)^k$. We wish to compute the probability $\Pr(A)$
over the sample returned by the algorithm given input set
$\{\x_1,\dots,\x_k\}$ and sampling size $k-1$. Note that since the sample
$\x_1,\dots,\x_k$ is symmetric under permutations, the probability of
$A$ should not depend on which index $i$ is selected in line 5 of
Algorithm~\ref{alg:standard}, so we have
\begin{align*}
\Pr(A) &= k\ \Pr(A\ |\  \text{Alg. \ref{alg:standard} selected }i\!=\!k)\\
&\propto \E_{\dxk}\bigg[\,\one_{A'}\,q_k(\X)
  \det\!\big(\X^\top\X\big) \bigg]\\
         &\propto\E_{\dxk}\bigg[\,\one_{A'}\,\frac{\det(\X_{-k}^\top\X_{-k})}{\det(\X^\top\X)}
           \det\!\big(\X^\top\X\big)\bigg]\\
  &=\E_{\dxk}\bigg[\,\one_{A'}\,\det(\X_{-k}^\top\X_{-k})\bigg]\\
  &\propto
    \vskxm(A),
\end{align*}
where in the above we skipped constant factors, since
they fall into the normalization constant. The lemma now follows by
induction over increasing $k$.
\end{proofof}

\end{document}